\documentclass{article}

\PassOptionsToPackage{numbers}{natbib}
\usepackage[nonatbib,final]{nips_2017}

\usepackage[utf8]{inputenc} 
\usepackage[T1]{fontenc}    
\usepackage{hyperref}       
\usepackage{url}            
\usepackage{booktabs}       
\usepackage{amsfonts}       
\usepackage{nicefrac}       
\usepackage{microtype}      

\usepackage{amsmath,amsfonts,amssymb,amsthm}
\usepackage{graphicx}
\usepackage{wrapfig}
\usepackage{wasysym}
\usepackage{enumitem}
\usepackage{epstopdf}

\usepackage[percent]{overpic}

\usepackage{color}
\usepackage{xcolor}
\definecolor{orange}{rgb}{0,0,0.6}

\newcommand{\myparagraph}[1]{\noindent\textbf{#1.}}

\newtheorem{lmm}{Lemma}
\newtheorem{prop}{Proposition}
\newtheorem{remark}{Remark}

\newtheorem{cor}{Corollary}

\def\Ber{\text{Ber}}
\def\X{\mathbf{X}}
\def\Y{\mathbf{Y}}
\def\U{\mathbf{U}}
\def\V{\mathbf{V}}
\def\u{\mathbf{u}}
\def\v{\mathbf{v}}
\def\r{\mathbf{r}}
\def\E{\mathbb{E}}
\def\Var{\mathbb{V}}
\def\S{\mathcal{S}}
\def\1{\mathbf{1}}
\def\0{\mathbf{0}}

\newcommand{\minimize}[1]{\underset{\displaystyle #1}{\rm minimum} \;}
\newcommand{\diag}[1]{{\rm diag}(#1)}

\title{Dropout Criterion and Matrix Factorization}
\title{An Analysis of Dropout for Matrix Factorization}

\author{
  {Jacopo Cavazza$^1$}\thanks{Work accomplished while visiting Johns Hopkins University, Center for Imaging Science}$\;$, Connor Lane$^2$, Benjamin D. Haeffele$^2$, Vittorio Murino$^1$, Ren\'{e} Vidal$^2$ \\\\
  $^1$ Pattern Analysis and Computer Vision,
  Istituto Italiano di Tecnologia,
  Genova, 16163, Italy \\
  $^2$ Johns Hopkins University, Center for Imaging Science, Baltimore, MD 21218, USA \\ \\  \texttt{jacopo.cavazza@iit.it}, \texttt{connor.lane@jhu.edu}, \texttt{bhaeffele@jhu.edu}, \\ \texttt{vittorio.murino@iit.it}, \texttt{rvidal@cis.jhu.edu} 
}

\begin{document}

\maketitle

\begin{abstract}
Dropout is a simple yet effective algorithm for regularizing neural networks by randomly dropping out units through Bernoulli multiplicative noise, and for some restricted problem classes, such as linear or logistic regression, several theoretical studies have demonstrated the equivalence between dropout and a fully deterministic optimization problem with data-dependent Tikhonov regularization. This work presents a theoretical analysis of dropout for matrix factorization, where Bernoulli random variables are used to drop a factor, thereby attempting to control the size of the factorization. While recent work has demonstrated the empirical effectiveness of dropout for matrix factorization, a theoretical understanding of the regularization properties of dropout in this context remains elusive. This work demonstrates the equivalence between dropout and a fully deterministic model for matrix factorization in which the factors are regularized by the sum of the product of the norms of the columns. While the resulting regularizer is closely related to a variational form of the nuclear norm, suggesting that dropout may limit the size of the factorization, we show that it is possible to trivially lower the objective value by doubling the size of the factorization. We show that this problem is caused by the use of a fixed dropout rate, which motivates the use of a rate that increases with the size of the factorization. Synthetic experiments validate our theoretical findings.
\end{abstract}

\section{Introduction}

Dropout \cite{DropoutCORR,DropoutJMLR} is a popular algorithm for training neural networks designed to prevent overfitting and circumvent performance degradation while shifting from training/validation to testing. During dropout training, for each example/mini-batch, neural units are randomly suppressed from the network with probability $1 - \theta$. Mathematically, this is equivalent to sampling, for each unit, a Bernoulli random variable $r\sim\Ber(\theta)$ and suppressing that unit if and only if $r = 0$. This generates a sub-network sampled from the original one, whose weights are updated through a backpropagation step, while the weights of the suppressed units are left unchanged. Then, when a new example/mini-batch is processed, new Bernoulli random variables are sampled to generate a new subnetwork and a new set of suppressed units. Since all the sub-networks are sampled from the original architecture, the weights are shared and dropout can be interpreted as a model ensemble. Interestingly, it has been proven that the (weighted geometric) average prediction of all the subnetworks can be efficiently computed by a single forward step involving the full network whose weights are scaled by $\theta$ \cite{DropoutJMLR,Baldi1,Baldi2}.


Recently, significant efforts have been made to understand the theoretical properties of dropout as an implicit regularization scheme \cite{Wager:NIPS13,Baldi1,Baldi2,Gal2016Dropout}. While in principle dropout is a stochastic training method based on randomly suppressing units from the network architecture, recent work has demonstrated its equivalence to a deterministic training scheme based on minimizing a loss augmented with data-dependent regularization. Although this equivalence requires a Taylor \cite{Wager:NIPS13,Baldi1,Baldi2} or Bayesian \cite{Gal2016Dropout} approximation to hold, these results have explained many properties of the regularizer induced by dropout. Indeed, it has been shown that dropout induces a non-monotone and non-convex function, which is sometimes divergent as a function of the network weights \cite{Wager:NIPS13,JMLR:v16:helmbold15a}. It is also known how dropout can handle model uncertainty in deep networks \cite{Gal2016Dropout} and when the dropout-regularized optimization problem yields a unique minimizer \cite{JMLR:v16:helmbold15a,Baldi2}. However, such a general understanding of dropout is often obtained by restricting the analysis to very simple models, such as linear or logistic regression \cite{Wager:NIPS13,NIPS2014_5502,Baldi1,Baldi2}.

The goal of this work is to provide a theoretical analysis of dropout in the context of matrix factorization. Given a fixed $m \times n$ matrix $\mathbf{X}$, the task is to find factors $\mathbf{U}$ and $\mathbf{V}$ of dimensions $m \times d$ and $n \times d$, respectively, such that $\mathbf{X} \approx \mathbf{U} \mathbf{V}^\top$ for some $d \geq \rho(\mathbf{X}) := {\rm rank}(\mathbf{X})$. In this context, applying the dropout criterion means to sample a $d$-dimensional random vector $\mathbf{r} = [r_1,\dots,r_d]$ with i.i.d. Bernoulli entries $r_i\sim\Ber(\theta)$ and to approximate $\mathbf{X}$ as $\mathbf{U}{\rm diag}(\mathbf{r})\mathbf{V}^\top$, where for any $j = 1,\dots, d$, the $j$-th columns of $\mathbf{U}$ and $\mathbf{V}$ are suppressed if $r_j = 0$ and left unchanged otherwise. 

To date, the use of dropout in matrix factorization has been investigated primarily from an empirical perspective \cite{Zhai:CoRR15,He2016}. Indeed, the analysis of \cite{He2016} is primarily empirical and \cite{Zhai:CoRR15} only develops a formal analogy between matrix factorization and a shallow encoder-decoder in order to combine the two and boost performance. To the best of our knowledge, the theoretical understanding of the regularization properties induced by dropout in matrix factorization remains largely an open problem.

\myparagraph{Paper contributions}
In this work, we study the theoretical properties of dropout in the context of matrix factorization. We first show that dropout regularization induces an equivalent deterministic optimization problem with regularization on the factors. Specifically, we show that the expected loss $\E_{\r} \| \X - \frac{1}{\theta}\U\diag \r \V^\top\|_F^2$ is equal to the regular loss $\| \X - \U \V^\top\|_F^2$ augmented with the regularizer $\Omega(\U,\V) = \sum_{k=1}^d \|\u_k\|_2^2 \|\v_k\|_2^2$, scaled by $\frac{1-\theta}{\theta}$, where $\u_k$ and $\v_k$ denote the $k$th columns of $\U$ and $\V$, respectively. This result provides an immediate interpretation for dropout in the traditional setting of factorization with a \emph{fixed} size of the factors, $d$. It is important to note, however, that in the case of matrix factorization, the number of columns in $\U$ and $\V$, $d$, is a model design parameter that must be either specified \textit{a priori} or learned in some way.  As the overall goal of dropout regularization is to prevent model over-fitting and constrain the degrees of freedom in the model, we also consider the case where the value of $d$ is learned directly from the data via an induced dropout regularization.

In the more complex case were $d$ is allowed to vary, the form of the regularizer $\Omega$ is very similar to the one used in the variational form of the nuclear norm, $\sum_{k=1}^d \|\u_k\|_2 \|\v_k\|_2$, suggesting that dropout could be used to induce low-rank factorizations (and hence limit the value of $d$). However, our analysis shows that when the dropout rate $1-\theta$ is independent of $d$, dropout regularization does nothing to constrain the size of the factorization (and in fact promotes factorizations with large numbers of columns in $\U$ and $\V$). 
This leads us to propose a novel adaptive dropout strategy in which the dropout rate increases with $d$ to bound the size of the factorization and learn the appropriate factorization size directly from the data.   
In particular, the contributions of this work include the following:
\begin{enumerate}[leftmargin=*]
	\item We analyze the regularization term induced by dropout when applied to matrix factorization with the squared Frobenius loss and derive an equivalent optimization problem where the same loss function is now regularized with a non-convex function. 
Additionally, our analysis also considers the case where the number of columns, $d$, is allowed to be variable and learned directly from the data. 
	
	\item We show that for a fixed dropout rate $1-\theta$, the regularizer induced by dropout does not control the size of the factorization, and in fact promotes solutions with a large value of $d$. We propose to solve this issue by using an adaptive choice for $\theta$ that depends on $d$.

	\item We show that the proposed variable dropout rate that scales with $d$ induces a pseudo-norm on the product of the factors $\mathbf{U} \mathbf{V}^\top$ which limits the rank of the factorization and show that the convex envelope of the induced pseudo-norm is equal to the squared nuclear norm of $\mathbf{U}\mathbf{V}^\top$.
	
	\item Numerical simulations validate the equivalence between the original dropout problem and its equivalent deterministic counterpart. We also demonstrate that our proposed variable dropout rate strategy correctly recovers low-rank matrices corrupted with noise, whereas dropout regularization with a fixed dropout rate does not. 
\end{enumerate}

\myparagraph{Paper outline} The remainder of the paper is organized as follows. In Section \ref{sez:rw} we briefly review the literature related to dropout. Sections \ref{sez:theory}, \ref{sez:theory2} and \ref{sez:theory3} present our theoretical analysis of the dropout criterion for matrix factorization. Our findings are supported by numerical simulations in Section \ref{sez:sim} and concluding remarks are given in Section \ref{sez:end}. 

\section{Related Work}\label{sez:rw}

The origins of dropout can be traced back to the literature on learning representations from input data corrupted by noise \cite{Bishop:NC95,Bengio:ICML09,Rifai:CoRR11}, and since the original formulation \cite{DropoutCORR,DropoutJMLR}, many algorithmic variations have been proposed \cite{ImprovedDropout,FastDropoutRecurrent,DropConv,AdaptiveDropout,AnnealedDropout,2016arXiv161101353A,CurriculumDropout}. Further, the empirical success of dropout for neural network training has motivated several works to investigate its formal properties from a theoretical point of view.
Wager et al. \cite{Wager:NIPS13} analyze dropout applied to the logistic loss for fitting $(x, y)$ data pairs where the distribution of $y$ given $x$ is described by a generalized linear model. By means of a Taylor approximation, they show that dropout induces a regularizer that depends on $x$ but not on $y$. Following on this line of work, Hembold and Long \cite{JMLR:v16:helmbold15a} discuss mathematical properties of the dropout regularizer (such as non-monotonicity and non-convexity) and derive a sufficient condition to guarantee a unique minimizer for the dropout criterion. 
Baldi and Sadowski \cite{Baldi1,Baldi2} consider dropout applied to deep neural networks with sigmoid activations and prove that the weighted geometric mean of all of the sub-networks can be computed with a single forward pass.
Wager et al. \cite{NIPS2014_5502} investigate the impact of dropout on the generalization error in terms of the bias-variance trade-off. Specifically, they present a theoretical analysis of the benefits related to dropout training under a Poisson topic model assumption in terms of a more favorable bound on the empirical risk minimization.
Finally, Gal and Ghahramani \cite{Gal2016Dropout} endow neural networks with a Bayesian framework to handle uncertainty of the network's predictions and investigate the connections between dropout training and inference for deep Gaussian processes.

In the context of matrix factorization, only a few works have investigated the dropout criterion. He et al. \cite{He2016} leverage the formal analogy between matrix factorization and shallow neural networks, which inspires the use of dropout for regularization and results in a model with better generalization abilities. However, the benefits of this combined approach are only experimental and no theoretical analysis is provided.  The authors of \cite{Zhai:CoRR15} provide some theoretical analysis for dropout applied to matrix factorization, but only as an argument to unify matrix factorization and encoder-decoder architectures.  To the best of our knowledge, there exists no theoretical analysis of the properties of the implicit regularization performed by dropout training for matrix factorization. Our paper aims to fill this gap.


\section{Dropout Criterion and Matrix Factorization}
\label{sez:theory} 

Given a fixed $m \times n$ matrix $\mathbf{X}$, we are interested in the problem of factorizing $\mathbf{X}$ as the product $\mathbf{U} \mathbf{V}^\top$, where $\mathbf{U}$ is $m \times d$ and $\mathbf{V}$ is $n \times d$, for some $d \geq \rho(\mathbf{X}) := {\rm rank}(\mathbf{X})$. In order to apply dropout to matrix factorization, we consider a random vector $\r = [r_1,\dots,r_d]$ whose elements are distributed as $r_i \stackrel{\text{i.i.d.}}{\sim} \Ber(\theta)$ and write the \emph{dropout criterion} \cite{Wager:NIPS13,JMLR:v16:helmbold15a,Baldi1,Baldi2} as the following optimization problem. 
\begin{equation}\label{eq:P1}
\min_{\U,\V,d} \mathbb{E}_{\mathbf{r}} \left\| \mathbf{X} - \dfrac{1}{\theta} \mathbf{U} {\rm diag}(\mathbf{r}) \mathbf{V}^\top \right\|_F^2.
\end{equation}
Here $\| \cdot \|_F$ denotes the Frobenius norm of a matrix,
$\mathbb{E}_{\mathbf{r}}$ denotes the expected value with respect to $\mathbf{r}$ and the minimization is carried out over $\mathbf{U}$, $\mathbf{V}$ and $d$.  Recall that we allow the size of the factorization, $d$, to be variable and seek to learn it directly from the data via the dropout regularization.

To see why the minimization of the above criteria can be achieved by dropping out columns of $\U$ and $\V$, observe that when $d$ is fixed and we use a gradient descent strategy, the gradient of the expected value is equal to the expected value of the gradient. Therefore, if we choose a stochastic gradient descent approach in which the expected gradient at each iteration is replaced by the gradient for a fixed sample $\r$, we obtain
\begin{align}
\label{eq:SGD}
\begin{bmatrix} \U^{t+1} \\ \V^{t+1} \end{bmatrix} = 
\begin{bmatrix} \U^t \\ \V^t \end{bmatrix} + \frac{2\epsilon}{\theta}
\begin{bmatrix} 
(\X-\U^t\diag{\r^t}\V^{t\top})\V^t \\
(\X-\U^t\diag{\r^t}\V^{t\top})^\top\U^t\\
\end{bmatrix} \diag{\r^t},
\end{align}
where $\epsilon > 0$ is the step size. Therefore, at iteration $t$, the columns of $\U$ and $\V$ for which $r_i^t = 0$ are not updated, and the gradient update is only applied to the columns for which $r_i^t = 1$.  

In order to achieve a better understanding of the implication of such random suppressions of columns, we consider a more general setting corresponding to a variable parameter $d$. In such a case, as an alternative optimization procedure, consider taking the expected value of the objective first. Following prior work for least squares fitting \cite{DropoutJMLR}, logistic regression \cite{Wager:NIPS13,Baldi1,Baldi2} and encoder-decoder learning \cite{Zhai:CoRR15}, we can show that \eqref{eq:P1} is equivalent to
\begin{equation}\label{eq:P2}
\min_{\U,\V,d} \left[ \| \mathbf{X} - \mathbf{U}  \mathbf{V}^\top \|_F^2 + \dfrac{1 - \theta}{\theta} \sum_{k = 1}^d \| \mathbf{u}_k \|_2^2 \| \mathbf{v}_k \|_2^2 \right],
\end{equation}
where $\mathbf{u}_k \in \mathbb{R}^m$ and $\mathbf{v}_k \in \mathbb{R}^n$ denote the $k$-th column of $\mathbf{U}$ and $\mathbf{V}$, respectively, for $k = 1,\dots,d$. The equivalence comes from the following theoretical result.\footnote{Detailed proofs of all theoretical results are presented in the Supplementary Material.}
\begin{prop}\label{prop:=}
For arbitrary $\theta,\mathbf{U},\mathbf{V}$ and $\mathbf{X}$,
	\begin{equation}\label{eq:=}
	\mathbb{E}_{\mathbf{r}} \left\| \mathbf{X} - \dfrac{1}{\theta} \mathbf{U} {\rm diag}(\mathbf{r}) \mathbf{V}^\top \right\|_F^2 = \| \mathbf{X} - \mathbf{U}  \mathbf{V}^\top \|_F^2 + \dfrac{1 - \theta}{\theta} \sum_{k = 1}^d \| \mathbf{u}_k \|_2^2 \| \mathbf{v}_k \|_2^2. 
	\end{equation}
\end{prop}
\begin{proof}
Note that the well known equality $\E(a^2) = \E(a)^2 + \Var(a)$ for a scalar random variable $a$ can be extended to matrices as $\E(\|A\|_F^2) = \| \E(A) \|_F^2 + \1^\top \Var(A) \1$ as soon as the entries in $A$ are independent. Applying it to $A = \mathbf{X} - \frac{1}{\theta}\mathbf{U}\diag \r  \mathbf{V}^\top$, and noticing that $\E (\diag \r) = \theta I $, we obtain $\E(A) = \X - \U\V^\top$. Since
$\Var(A) \!\!=\!\! \frac{1}{\theta^2} \Var( \U\diag\r\V^\top)$ and $ \U\diag\r\V^\top = \sum_k \u_k\v_k^\top r_k$, we have
\begin{align}
\theta^2 \1^\top \Var(A) \1 = \sum_{ijk} \Var(u_{ik}v_{jk}r_k) =  \sum_{ijk} u_{ik}^2 v_{jk}^2 \Var(r_k)  =\theta(1-\theta) \sum_{k} \|\u_{k}\|_2^2 \|\v_{k}\|_2^2
\end{align}
due to the fact that $r_k$ are independent. This completes the proof.
\end{proof}

Therefore, the optimization problem \eqref{eq:P1} can be alternatively tackled by solving \eqref{eq:P2} where the same loss function (quadratic Frobenius norm) enforces $\mathbf{U}\mathbf{V}^\top$ to be close to $\mathbf{X}$. Moreover, notice that the random suppression of columns in $\mathbf{U}$ and $\mathbf{V}$ used in \eqref{eq:SGD} is replaced in \eqref{eq:P2} by the regularizer
\begin{equation}\label{eq:omega}
\Omega(\mathbf{U},\mathbf{V}) = \sum_{k = 1}^d \| \mathbf{u}_k \|_2^2 \| \mathbf{v}_k \|_2^2
\end{equation}
which is weighted by the factor $\frac{1 - \theta}{\theta}$, where $\theta$ the expected value of $r_1,\dots,r_d \sim \Ber(\theta)$. 

\begin{remark}
Notice that in \eqref{eq:=} we have to assume that $\theta \neq 0$ in order to avoid division by zero. This is not a problem because when $\theta = 0$ the probability of suppressing any columns in $\mathbf{U}$ and $\mathbf{V}$ will be $1$, resulting in a degenerate case. On the other hand, we can also disregard the case $\theta = 1$, in which no column is suppressed at all and it is trivial to verify that the left hand-side of \eqref{eq:=} coincides with the right-hand side. Thus, in what follows, we will assume $0 < \theta < 1$.
\end{remark}

\section{Connections with Nuclear Norm Minimization}\label{sez:theory2}
To give a better understanding of \eqref{eq:omega}, we first investigate its relationship with a popular regularizer for matrix factorization, namely the nuclear norm $\| \mathbf{Y} \|_\star$. Defined as the sum of the singular values of $\mathbf{Y}$, the nuclear norm is widely used as a convex relaxation of the matrix rank and can optimally recover low-rank matrices under certain conditions \cite{Recht:2010}. The connection between $\| \cdot \|_\star$ and \eqref{eq:omega} becomes clearer when considering the following variational form of the nuclear norm \cite{Srebro:NIPS2004,Rennie:2005}:
\begin{equation}\label{eq:Nstar}
\| \mathbf{Y} \|_\star = \inf_{d,\U,\V} \sum_{k = 1}^d \| \mathbf{u}_k \|_2 \| \mathbf{v}_k \|_2 
\quad \text{s.t.} \quad 
d \geq \rho( \mathbf{Y}), \U \in \mathbb{R}^{m\times d}, \V \in \mathbb{R}^{n\times d} ~ \text{and}~ \U\V^\top = \mathbf{Y}.
\end{equation}
This fact is used in \cite{Bach:2008,Bach:2013,HaeffeleYV14,Haeffele:2015} to show that the convex optimization problem $\min_{\mathbf{Y}} \|\X -  \mathbf{Y}\|_F^2 + \lambda \| \mathbf{Y}\|_\star$ is equivalent to the non-convex optimization problem 
\begin{equation}
\min_{\U,\V,d} \| \X - \U\V^\top\|_F^2 + \lambda \sum_{k = 1}^d \| \mathbf{u}_k \|_2 \| \mathbf{v}_k \|_2 
\quad \text{s.t.} \quad \U \in \mathbb{R}^{m\times d}, \V \in \mathbb{R}^{n\times d}
\end{equation}
in the sense that if $(\U,\V)$ is a local minimizer of the non-convex problem such that for some $k$ we have $\u_k = \0$ and $\v_k = \0$, then  $(\U,\V)$ is a global minimizer of the non-convex problem and $\mathbf{Y} = \U\V^\top$ is a global minimizer of the convex problem.

But what does the variational form of the nuclear norm tells us about the regularizer $\Omega$ in \eqref{eq:omega} induced by dropout? Notice the extreme similarity between the functional optimized in \eqref{eq:Nstar} and \eqref{eq:omega}: the only difference is that the Euclidean norms of the columns of $\U$ and $\V$ are squared in \eqref{eq:omega}. Naively, one can argue that such difference is extremely marginal and therefore interpret dropout for matrix factorization as an unexpected way to achieve nuclear norm regularization on the factorization.

However, this is \emph{not} the case. To see this, inspired by the variational form of the nuclear norm in \eqref{eq:Nstar}, let us consider the following optimization problem:
\begin{equation}
\inf_{d,\U,\V} \sum_{k = 1}^d \| \mathbf{u}_k \|_2^2 \| \mathbf{v}_k \|_2^2 
\quad \text{s.t.} \quad 
d \geq \rho( \mathbf{Y}), \U \in \mathbb{R}^{m\times d}, \V \in \mathbb{R}^{n\times d} ~ \text{and}~ \U\V^\top = \mathbf{Y}.
\end{equation}
Suppose that we are given any set of factors $\overline{\mathbf{U}}$ and $\overline{\mathbf{V}}$, both with $\overline{d}$ columns, such that $\overline{\mathbf{U}}~ \overline{\mathbf{V}}^\top = \Y$. Then, we can construct a pair of matrices $\mathbf{A} = \frac{\sqrt{2}}{2} [\overline{\mathbf{U}},\overline{\mathbf{U}}] \in \mathbb{R}^{m \times 2\overline{d}}$ and $\mathbf{B} = \frac{\sqrt{2}}{2} [\overline{\mathbf{V}},\overline{\mathbf{V}}] \in \mathbb{R}^{n \times 2\overline{d}}$ such that $\mathbf{A}\mathbf{B}^\top = \Y$. However, observe that $\Omega(\mathbf{A},\mathbf{B}) = \frac{1}{2} \Omega(\mathbf{U},\mathbf{V})$, which implies that the regularizer $\Omega$ does not penalize the size of the factorization. On the contrary, it encourages factorizations with a large number of columns, as we can always reduce the value of $\Omega$ by increasing the number of columns, which provides the main argument to prove the following proposition.
\begin{prop}\label{prop:zero}
	The infimum of the regularizer $\Omega$ in \eqref{eq:omega} is equal to zero, i.e.,
	\begin{equation}\label{eq:zero}
0 = \inf_{d,\U,\V} \sum_{k = 1}^d \| \mathbf{u}_k \|_2^2 \| \mathbf{v}_k \|_2^2 
\quad \text{s.t.} \quad 
d \geq \rho(\X), \U \in \mathbb{R}^{m\times d}, \V \in \mathbb{R}^{n\times d} ~ \text{and}~ \U\V^\top = \X.
	\end{equation} 
\end{prop} 

As a consequence, using dropout for matrix factorization does nothing to limit the size of the factorization (i.e., limit the number of columns, $d$), due to the fact that the optimization problem solved by dropout \eqref{eq:P1} is equivalent to the regularized factorization problem \eqref{eq:P2}, which is always reduced in value by increasing the number of columns in $(\mathbf{U},\mathbf{V})$.

\section{Matrix Dropout with Adaptive Dropout Rate}\label{sez:theory3}

As discussed in the previous section, a key drawback of the regularizer $\Omega$ is that when $d$ is increased the value of $\Omega$ is decreased (for example, $d \to 2d$ results in $\Omega \to \Omega/2$). In order to compensate for this drawback, we replace the fixed choice for $\theta$ in \eqref{eq:P2} with an adaptive parameter $\theta(d)$, $d \in \mathbb{N} \setminus \{ 0 \}$, so that the weighting factor in \eqref{eq:P2}, $\frac{1 - \theta(d)}{\theta(d)}$, increases as $d$ increases. Specifically, we are interested in defining a function $\theta = \theta(d)$ such that the weighting factor $\tfrac{1-\theta(d)}{\theta(d)}$ in \eqref{eq:P2} grows linearly with $d$,
\begin{equation}\label{eq:base}
\dfrac{1-\theta(kd)}{\theta(kd)} = k \dfrac{1-\theta(d)}{\theta(d)} \ \ \forall \ k \in \mathbb{N}.
\end{equation}

To accomplish this, given any $\bar{\theta}$ such that $0 < \bar{\theta} < 1$ we define $\theta(d)$ as
\begin{equation}\label{eq:thetad}
	\theta \colon \mathbb{N} \setminus \{ 0\} \rightarrow \mathbb{R}, \qquad \theta(d) = \dfrac{\overline{\theta}}{d - (d - 1)\overline{\theta}.}
\end{equation}
and note that $\theta(d)$ satisfies the following proposition.
\begin{prop}\label{prop:adapt}
	
	For $\theta(d)$ as defined in \eqref{eq:thetad}, the following properties hold.
	\begin{enumerate}
		\item $0 < \theta(d) < 1$ for all $d \in \mathbb{N} \setminus \{ 0\}$. 
		\item $\dfrac{1 - \theta(kd)}{\theta(kd)} = k \dfrac{1-\theta(d)}{\theta(d)}$ for all $k \in \mathbb{N} \setminus \{ 0 \}$.
	\end{enumerate}
\end{prop}

The definition of $\theta(d)$ in \eqref{eq:thetad} induces an adaptive scheduling for the parameter $\theta$, which is determined by the parameter $\theta(1) = \overline{\theta}$. The idea of introducing an adaptive value for the probability of retaining units in dropout training for neural networks has been explored by Ba and Frey \cite{AdaptiveDropout}, Rennie et al. \cite{AnnealedDropout} and Morerio et al. \cite{CurriculumDropout}. However, these prior works typically adjust the dropout rate based on the values of the output from a previous later \cite{AdaptiveDropout} or based on the number of backpropagation's epochs \cite{CurriculumDropout,AnnealedDropout}. Here, in contrast, we are selecting a different value for $\theta$ as a function of the size of the factorization we are searching, or, put in the terms of neural networks, the dropout rate is modulated based on the number of units in the network.

Given this proposed modification to the dropout rate, we define $\lambda_d$ as
\begin{equation}
	\lambda_d = \dfrac{1-\theta(d)}{\theta(d)}
\end{equation}
and now propose a modified version of \eqref{eq:P2} given by
\begin{equation}\label{eq:P3}
\min_{\U,\V,d} \left[ \| \mathbf{X} - \mathbf{U}  \mathbf{V}^\top \|_F^2 + \lambda_d \sum_{k = 1}^d \| \mathbf{u}_k \|_2^2 \| \mathbf{v}_k \|_2^2 \right].
\end{equation}

Taking advantage of $\lambda_d$ as in Proposition \ref{prop:adapt}, we can now correct the bias of \eqref{eq:omega} in promoting over-sized factorizations (see Section \ref{sez:sim}) by constructing a regularizer based on the value of $\lambda_d \Omega(\U,\V)$. In addition, we can guarantee strong formal properties of the regularizer which naturally induces a quasi-norm on $m \times n$ matrices.  In particular, we note the following result.

\begin{prop}
	With the previous notation, for any $m \times n$ matrix $\mathbf{Y}$, let
	\begin{equation}\label{eq:norm}
	\| \mathbf{Y} \|_{\hspace{-.5 mm} \vartriangle} = \min_{d,\U,\V} \sqrt{ \lambda_{d} \sum_{k = 1}^d \| \mathbf{u}_k \|_2^2 \| \mathbf{v}_k \|_2^2} \quad \text{s.t.} \quad 
	d \geq \rho(\Y), \U \in \mathbb{R}^{m\times d}, \V \in \mathbb{R}^{n\times d} ~ \text{and}~ \U\V^\top = \Y.
	\end{equation}
	Then, \eqref{eq:norm} defines a quasi-norm over $m \times n$ matrices, i.e., $\| \Y \|_{\vartriangle}$ satisfies:
	\begin{align}
	& \| \Y \|_{\vartriangle} \geq 0 \qquad \mbox{for every} \; \Y \in \mathbb{R}^{m \times n} \label{eq:1} \\
	& \| \Y \|_{\vartriangle} = 0  \iff \Y = \0 \label{eq:2} \\
	& \| \alpha \Y \|_{\vartriangle} = |\alpha|\|  \Y \|_{\vartriangle} \qquad \mbox{for every} \; \alpha \in \mathbb{R} \; \mbox{and} \; \Y \in \mathbb{R}^{m \times n} \label{eq:3} \\
	& \exists \mbox{ $C >0$ (in particular $C=\sqrt{2}$) such that} \; \| \mathbf{Y} + \mathbf{Z} \|_{\vartriangle} \leq C (\|  \mathbf{Y} \|_{\vartriangle} + \| \mathbf{Z} \|_{{\hspace{-.5 mm} \vartriangle}}) \; \forall (\mathbf{Y}, \mathbf{Z}) 
	\label{eq:4}
	\end{align}
\end{prop}

Here we note that $\|\Y\|_\vartriangle^2$ is precisely the regularization induced in \eqref{eq:P3} by our variable choice of $\theta(d)$ in \eqref{eq:thetad}.  To further motivate the adaptive dropout rate, we also prove the following result which shows that even though the $\|\Y\|_\vartriangle$ function is not necessarily a convex function on $\Y$ (due to the fact that the triangle inequality is only shown for a constant $C>1$), the convex envelope of the induced regularization is equivalent to \emph{squared} nuclear norm regularization.

\begin{prop}
\label{prop:convex_env}
The convex envelope of $\tfrac{1}{2} \| \Y \|^2_\vartriangle$ is $\tfrac{1-\bar{\theta}}{2 \bar{\theta}} \| \Y\|^2_\star$.
\end{prop}

This result suggests that the regularization induced by our adaptive dropout rate scheme acts as a regularization on the rank of the factorization and is likely a tighter bound on the matrix rank than the fully convex relaxation to the nuclear norm. Notice also that the convex envelope is given by the square of the nuclear norm, as intuitively expected since the definition of $\Omega$ has the square of the norms of the columns of $\U$ and $\V$. Interestingly, the matrix approximation with squared nuclear norm regularization is not used in typical formulations, and it admits a closed form solution, as stated in the next proposition.

\begin{prop}
\label{prop:nuc_squared}
Let $\X = \mathbf{L} \boldsymbol{\Sigma} \mathbf{R}^\top$ be the singular valued decomposition of $\X$. The optimal solution to
\begin{equation}
\label{prop:num_squared}
\min_{\Y} \quad  \| \X - \Y \|_F^2 + \lambda \| \Y \|_\star^2,
\end{equation}
is given by $\Y = \mathbf{L} \S_{\mu}(\boldsymbol{\Sigma}) \mathbf{R}^\top$, where $\lambda >0$, $\mu = \frac{\lambda d}{1 + \lambda d} \bar\sigma_d(\X)$, $\bar\sigma_d(\X)$ is the average of the top $d$ singular values of $\X$, $d$ denotes the largest integer such that $\sigma_d(\X) > \frac{\lambda d}{1+\lambda d} \bar\sigma_d(\X)$, and $\S_{\mu}$ defines the shrinkage thresholding operator \cite{Vidal:book} applied to the singular values of $\X$.
\end{prop}

In conclusion, despite the regularizer $\Omega(\U,\V)$ paired with a fixed value of $\theta$ can not be directly linked with $\| \X \|_\star$ due to Proposition \ref{prop:zero}, Proposition \ref{prop:nuc_squared} prospects an unexpected connection between the optimization problem \eqref{eq:P2} and the squared nuclear norm regularization when an adaptive choice for $\theta = \theta(d)$ is adopted. Such finding will be corroborated by numerical evidences in the next Section. 

\section{Numerical Simulations}\label{sez:sim}

To demonstrate our predictions experimentally, we first verify the equivalence between the stochastic \eqref{eq:P1} and deterministic \eqref{eq:P2} formulations of matrix factorization dropout by 
constructing a synthetic data matrix $\mathbf{X}$, where $m = n = 100$, defined as the matrix product $\mathbf{X} = \mathbf{U}_0 {\mathbf{V}_0}^\top$ where $\mathbf{U}_0, \mathbf{V}_0 \in \mathbb{R}^{100 \times d}$ with $d=10,40,160$. The entries of $\mathbf{U}_0$ and $\mathbf{V}_0$ were sampled from a zero-mean Gaussian distribution with standard deviation 0.1. Both the stochastic \eqref{eq:P1} and deterministic \eqref{eq:P2} formulations of dropout were solved by 10,000 iterations of gradient descent with diminishing $O(\frac{1}{t})$ lengths for the step size. In the stochastic setting, we approximate the objective in \eqref{eq:P1} and the gradient by sampling a new Bernoulli vector $\mathbf{r}$ for every iteration of the optimization as in \cite{DropoutCORR,DropoutJMLR}.

\begin{figure}
	\includegraphics[width=\textwidth,keepaspectratio]{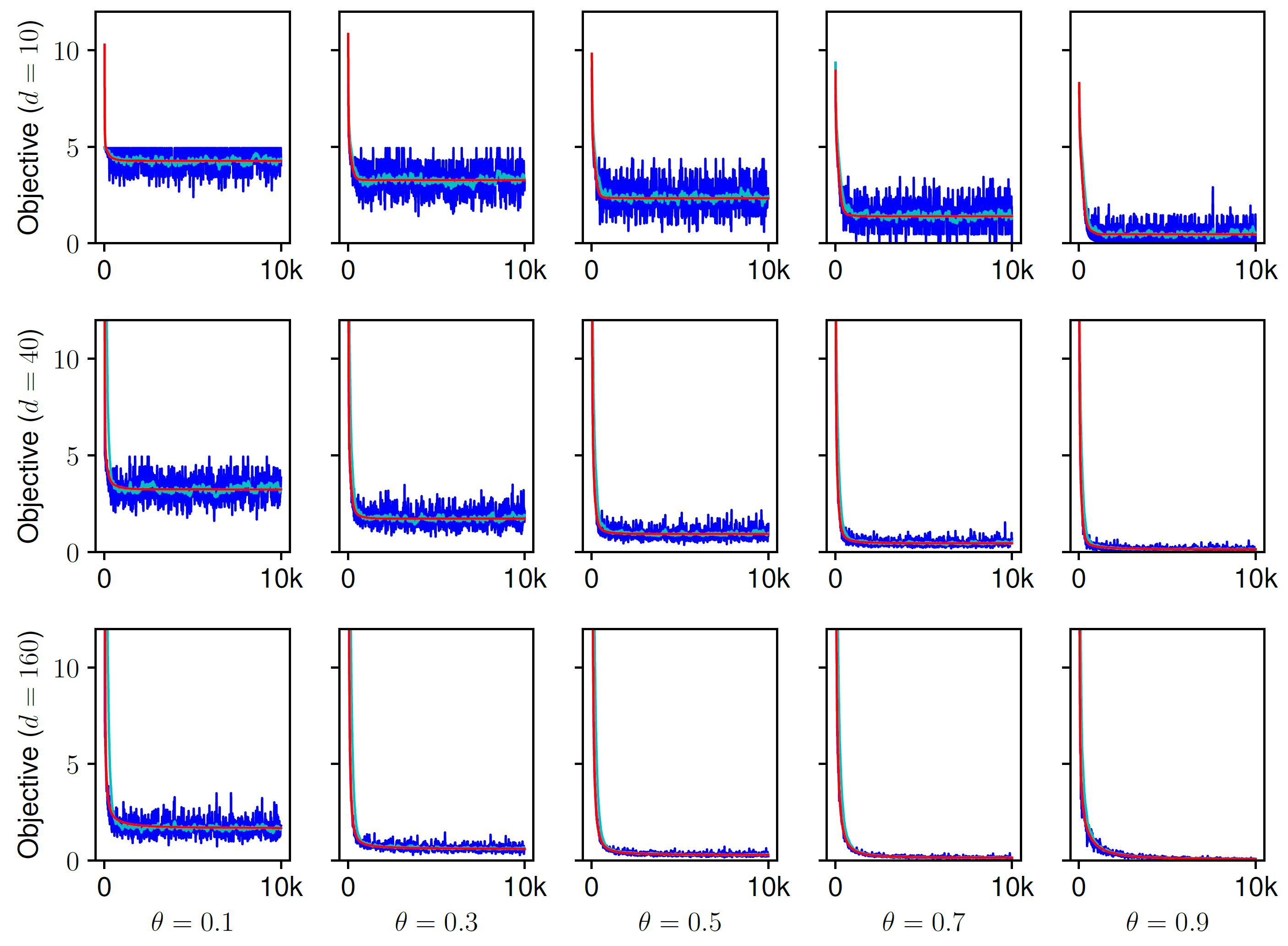}
  \caption{For $\theta \in \{0.1, 0.3, 0.5, 0.7, 0.9\}$ and $d \in \{10,40,160\}$ we compare the deterministic problem \eqref{eq:P2} (red) with its stochastic counterpart \eqref{eq:P1} (blue). The exponential moving average of the stochastic objective is shown in cyan. Best viewed in color.}
	\label{fig:obj_curves}
\end{figure}

Figure \ref{fig:obj_curves} plots the objective curves for the stochastic and deterministic dropout formulations for different choices of the dropout rate $\theta = 0.1, 0.3, 0.5, 0.7, 0.9$ and factorization size $d = 10, 40, 160$. We observe that across all choices of parameters $\theta$ and $d$, the deterministic objective \eqref{eq:P2} tracks the apparent expected value of its stochastic counterpart \eqref{eq:P1}. This provides experimental evidence for the fact that the two formulations are equivalent, as predicted.

Having verified the equivalence between \eqref{eq:P1} and \eqref{eq:P2}, we are now interested in supporting our theoretical analysis of the regularizer \eqref{eq:omega} through a numerical simulations.  Specifically, we investigate the rank-limiting effects of the three regularization schemes considered: matrix factorization dropout with a fixed value of $\theta$, adaptive dropout with a value of $\theta(d)$ that scales with the dimension of the factors, and the convex, nuclear-norm squared problem which is the convex envelope of the problem induced by our proposed adaptive dropout scheme. We hypothesize that the adaptive dropout scheme should promote low-rank factorizations, while unmodified dropout should not. Moreover, in view of Proposition \ref{prop:convex_env}, we evaluate whether adaptive dropout and the nuclear-norm squared formulation produce similar solutions.

\begin{figure}[h]
	\includegraphics[width=\columnwidth,keepaspectratio]{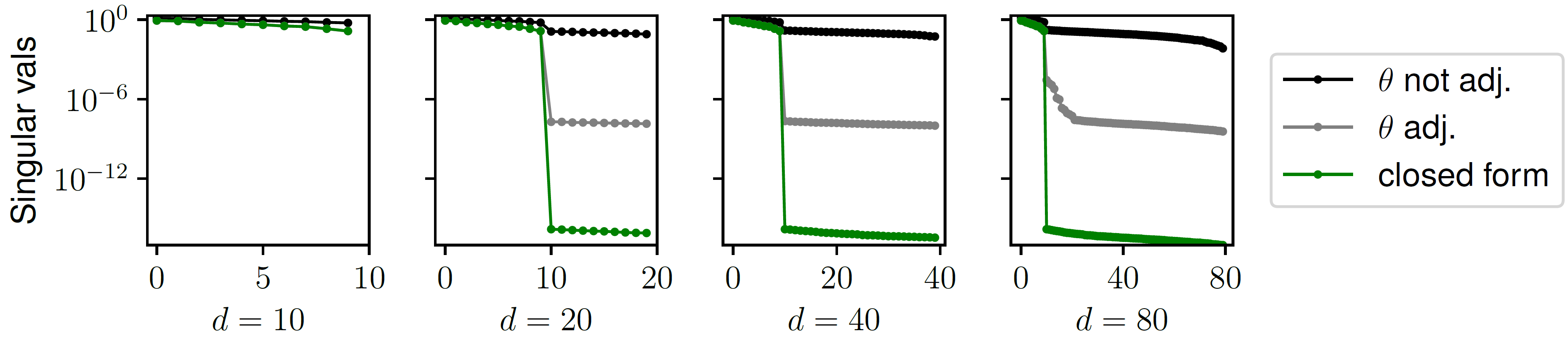}
  \caption{Singular values corresponding to the optimal solutions of the three regularization schemes considered: fixed dropout rate of $0.9$ (black), adaptive dropout $\theta = \theta(d)$ as \eqref{eq:thetad} with $\overline{\theta} = 0.9$ (gray), and the nuclear-norm squared closed-form optimization as in Proposition \ref{prop:nuc_squared} (green). The fixed dropout regularization has little effect as the size of the factorization $d$ increases. Adaptive dropout and nuclear-norm squared regularization both result in similar degrees of shrinkage-thresholding. Note that the singular values for the nuclear-norm squared case do not change with $d$. Best viewed in color.}
	\label{fig:rank_plots}
\end{figure}

We constructed a synthetic dataset $X$ consisting of a low-rank matrix combined with dense Gaussian noise. Specifically, we let $X = U_0 V_0^\top + Z_0$ where $U_0, V_0 \in \mathbb{R}^{100 \times 10}$ contain entries drawn from a normal distribution ($\mu=0$, $\sigma = 0.1$), as before. The entries of the noise matrix $Z_0$ were drawn from a normal distribution with $\sigma = 0.01$. We fixed the dropout parameter $\bar{\theta} = 0.9$ and solved the dropout optimization using gradient descent as described previously while using the closed form solution given by Proposition \ref{prop:nuc_squared} to solve the problem with nuclear-norm squared regularization.

Figure \ref{fig:rank_plots} plots the singular values for the optimal solution to each of the three problems. We observe first that without adjusting $\theta$, dropout regularization has little effect on the rank of the solution.  The smallest singular values are still relatively high and not modified significantly compared to the singular values of the original data. On the other hand, by adjusting the dropout rate based on the size of the factorization we observe that the method correctly recovers the rank of the noise-free data which also closely matches the predicted convex envelope with the nuclear-norm squared regularizer (note the log scale of the singular values). Furthermore, across the choices for $d$, the relative Frobenius distances between the solutions of these two methods are very small (between $10^{-6}$ and $10^{-2}$). Taken together, our theoretical predictions and experimental results suggest that adapting the dropout rate based on the size of the factorization is critical to ensuring the effectiveness of dropout as a regularizer and in limiting the degrees of freedom of the model.


\section{Conclusions}\label{sez:end}

Here we have presented a theoretical analysis of dropout as a potential regularization strategy in matrix factorization problems and shown that the stochastic dropout formulation induces a deterministic regularization on the matrix factors.  Additionally, we demonstrated that using dropout with a fixed dropout rate is not sufficient to limit the size of the factorization.  Instead, we proposed a dropout strategy that adjusts the dropout rate based on the size of the factorization which mediates this problem and results in an induced regularization that is closely related to the squared nuclear norm.  Finally, we presented experimental results that confirmed our theoretical predictions.  While we have focused primarily on matrix factorization in this paper, our analysis is easily extended to many forms of neural network training that employ dropout on a final, fully-connected layer, which we save for future work.


\bibliographystyle{plain}
\bibliography{dropout}

\appendix

\section*{Supplementary Material}
	
	\subsection*{Proofs from Section 3: Dropout Criterion and Matrix Factorization}
	
	For a fixed $m \times n$ matrix $\mathbf{X}$, consider the problem of factorizing $\mathbf{X}$ into the product $\mathbf{U} \mathbf{V}^\top$ where $\mathbf{U}$ is $m \times d$ and $\mathbf{V}$ is $n \times d$, for some $d \geq \rho(\mathbf{X}) := {\rm rank}(\mathbf{X})$. 
	
	\begin{prop}
		Define $\mathbf{r} = [r_1,\dots,r_d]$, whose elements are ${\rm Bernoulli}(\theta)$ i.i.d. where $0 < \theta < 1$. Furthermore, denote $\mathbf{u}_k \in \mathbb{R}^m$ and $\mathbf{v}_k \in \mathbb{R}^n$ the $k$-th column in $\mathbf{U}$ and $\mathbf{V}$, respectively, $k = 1,\dots,d$. Then, 
		\begin{equation}
		\mathbb{E}_{\mathbf{r}} \left\| \mathbf{X} - \dfrac{1}{\theta} \mathbf{U} {\rm diag}(\mathbf{r}) \mathbf{V}^\top \right\|_F^2 = \| \mathbf{X} - \mathbf{U}  \mathbf{V}^\top \|_F^2 + \dfrac{1 - \theta}{\theta} \sum_{k = 1}^d \| \mathbf{u}_k \|_2^2 \| \mathbf{v}_k \|_2^2. 
		\end{equation}
	\end{prop}
	
	\proof
	
	Equivalently, we will demonstrate that
	\begin{equation}\nonumber
	\mathbb{E}_{\mathbf{r}} \|\theta \mathbf{X} - \mathbf{U} {\rm diag}(\mathbf{r}) \mathbf{V}^\top \|_F^2 = \theta^2 \| \mathbf{X} - \mathbf{U}  \mathbf{V}^\top \|_F^2 + \theta( 1 - \theta ) \sum_{k = 1}^d \| \mathbf{u}_k \|_2^2 \| \mathbf{v}_k \|_2^2. 
	\end{equation}
	Since 
	\begin{align}
	&\mathbb{E}_{\mathbf{r}} \|\theta \mathbf{X} - \mathbf{U} {\rm diag}(\mathbf{r}) \mathbf{V}^\top \|_F^2 = \nonumber \\ &= \mathbb{E}_{\mathbf{r}} \left\| \begin{bmatrix} \theta X_{11} - \sum_{k = 1}^d U_{1k}r_kV_{1k}, & \dots, & \theta X_{1n} - \sum_{k = 1}^d U_{1k}r_kV_{nk} \\
	\vdots & \ddots & \vdots \\
	\theta X_{m1} - \sum_{k = 1}^d U_{mk}r_kV_{1k}, & \dots, & \theta X_{mn} - \sum_{k = 1}^d U_{mk}r_kV_{nk}
	\end{bmatrix} \right \|_F^2,
	\end{align}
	by definition of Frobenius norm and linearity of $\mathbb{E}_{\mathbf{r}}$, we elicit
	\begin{equation}
	\mathbb{E}_{\mathbf{r}} \|\theta \mathbf{X} - \mathbf{U} {\rm diag}(\mathbf{r}) \mathbf{V}^\top \|_F^2 = \sum_{i = 1}^m \sum_{j = 1}^n \mathbb{E}_{\mathbf{r}} \left[ \left( \theta X_{ij} - \sum_{k = 1}^d U_{ik} r_k V_{jk} \right)^2 \right].
	\end{equation}
	Use the bias-variance decomposition $\mathbb{E}[r^2] = \mathbb{V}[r] + \mathbb{E}[r]^2$, holding for a scalar random variable $r$.
	\begin{align}
	\mathbb{E}_{\mathbf{r}} \|\theta \mathbf{X} - \mathbf{U} {\rm diag}(\mathbf{r}) \mathbf{V}^\top \|_F^2 &  = \sum_{i = 1}^m \sum_{j = 1}^n \mathbb{V}_{\mathbf{r}} \left[ \theta X_{ij} - \sum_{k = 1}^d U_{ik} r_k V_{jk} \right] + \nonumber \\
	& + \sum_{i = 1}^m \sum_{j = 1}^n \left( \mathbb{E}_{\mathbf{r}} \left[  \theta X_{ij} - \sum_{k = 1}^d U_{ik} r_k V_{jk} \right] \right)^2.
	\end{align}
	Since $r_1,\dots,r_d$ are i.i.d., use the properties of expectation $\E_{\mathbf{r}}$ and variance $\Var_{\mathbf{r}}$ with respect to linear combinations of independent random variables.
	\begin{align}
	\mathbb{E}_{\mathbf{r}} \|\theta \mathbf{X} - \mathbf{U} {\rm diag}(\mathbf{r}) \mathbf{V}^\top \|_F^2 &  = \sum_{i = 1}^m \sum_{j = 1}^n \sum_{k = 1}^d U_{ik}^2 V_{jk}^2 \mathbb{V}_{\mathbf{r}} \left[  r_k  \right] + \nonumber \\
	& + \sum_{i = 1}^m \sum_{j = 1}^n \left(   \theta X_{ij} - \sum_{k = 1}^d U_{ik} \mathbb{E}_{\mathbf{r}} \left[ r_k \right] V_{jk}  \right)^2.
	\end{align}
	Exploit the analytical formulas for expected value and variance of a Bernoulli$(\theta)$ distribution.
	\begin{align}
	\mathbb{E}_{\mathbf{r}} \|\theta \mathbf{X} - \mathbf{U} {\rm diag}(\mathbf{r}) \mathbf{V}^\top \|_F^2 &  = \sum_{i = 1}^m \sum_{j = 1}^n \sum_{k = 1}^d U_{ik}^2 V_{jk}^2 \cdot \theta (1 - \theta) + \nonumber \\
	& + \sum_{i = 1}^m \sum_{j = 1}^n \left(   \theta X_{ij} - \sum_{k = 1}^d U_{ik} \cdot \theta \cdot V_{jk}  \right)^2.
	\end{align}
	Rearrange the terms.
	\begin{align}
	\mathbb{E}_{\mathbf{r}} \|\theta \mathbf{X} - \mathbf{U} {\rm diag}(\mathbf{r}) \mathbf{V}^\top \|_F^2 &  = \theta (1 - \theta) \sum_{k = 1}^d \left( \sum_{i = 1}^m U_{ik}^2 \right)  \left( \sum_{j = 1}^n V_{jk}^2 \right) + \nonumber \\
	& + \theta^2 \sum_{i = 1}^m \sum_{j = 1}^n \left(    X_{ij} - \sum_{k = 1}^d U_{ik} V_{jk}  \right)^2.
	\end{align}
	Use the definition of row-by-column product of matrices
	\begin{align}
	\mathbb{E}_{\mathbf{r}} \|\theta \mathbf{X} - \mathbf{U} {\rm diag}(\mathbf{r}) \mathbf{V}^\top \|_F^2 &  = \theta (1 - \theta) \sum_{k = 1}^d \left( \sum_{i = 1}^m U_{ik}^2 \right)  \left( \sum_{j = 1}^n V_{jk}^2 \right) + \nonumber \\
	& + \theta^2 \sum_{i = 1}^m \sum_{j = 1}^n \left(    X_{ij} - \left[\mathbf{U} \mathbf{V}^\top \right]_{ij}  \right)^2.
	\end{align}
	Apply the definitions of squared Euclidean norm $\| \cdot \|_2^2$ and Frobenius norm $\| \cdot \|_F$
	\begin{align*}
	\mathbb{E}_{\mathbf{r}} \|\theta \mathbf{X} - \mathbf{U} {\rm diag}(\mathbf{r}) \mathbf{V}^\top \|_F^2 &  = \theta (1 - \theta) \sum_{k = 1}^d \| \mathbf{u}_{k} \|_2^2 \| \mathbf{v}_{k} \|_2^2 + \theta^2 \| \mathbf{X} - \mathbf{U} \mathbf{V}^\top \|_F.
	\end{align*}
	This concludes the proof.\endproof
	
	
	\subsection*{Proofs from Section 4: Connections with Nuclear Norm Minimization}
	
	\begin{prop}\label{prop:zero}
		\begin{equation}\label{eq:zero}
		0 = \inf_{d,\U,\V} \sum_{k = 1}^d \| \mathbf{u}_k \|_2^2 \| \mathbf{v}_k \|_2^2 
		\quad \text{s.t.} \quad 
		d \geq \rho(\X), \U \in \mathbb{R}^{m\times d}, \V \in \mathbb{R}^{n\times d} ~ \text{and}~ \U\V^\top = \X.
		\end{equation} 
	\end{prop}
	
	\proof Let $\mathbf{U}$ and $\mathbf{V}$ such that $\mathbf{U} \mathbf{V}^\top = \mathbf{X}$ for a particular choice of $d$. Denote
	\begin{equation}
	\Omega(\mathbf{U},\mathbf{V}) = \sum_{k = 1}^d \| \mathbf{u}_k \|_2^2 \| \mathbf{v}_k \|_2^2
	\end{equation}
	and define
	\begin{align}
	\mathbf{A} & = \dfrac{\sqrt{2}}{2} [\mathbf{U},\mathbf{U}] \in \mathbb{R}^{m \times 2d} \\
	\mathbf{B} & = \dfrac{\sqrt{2}}{2} [\mathbf{V},\mathbf{V}] \in \mathbb{R}^{n \times 2d}.
	\end{align}
	Then
	\begin{equation}
	\mathbf{A} \mathbf{B}^\top = \left(\dfrac{\sqrt{2}}{2}\right)^2 \mathbf{U}\mathbf{V}^\top + \left(\dfrac{\sqrt{2}}{2}\right)^2 \mathbf{U}\mathbf{V}^\top = \dfrac{1}{2} \mathbf{X} + \dfrac{1}{2} \mathbf{X} = \mathbf{X} 
	\end{equation}
	and
	\begin{align}
	\Omega(\mathbf{A},\mathbf{B}) & = \sum_{k = 1}^{2d} \| \mathbf{a}_k \|_2^2 \| \mathbf{b}_k \|_2^2  \\ &= \dfrac{1}{4} \sum_{k = 1}^{d}  \| \mathbf{u}_k \|_2^2 \| \mathbf{v}_k \|_2^2 + \dfrac{1}{4} \sum_{k = 1}^{d}  \| \mathbf{u}_k \|_2^2 \| \mathbf{v}_k \|_2^2 = \dfrac{1}{2} \Omega(\mathbf{U},\mathbf{V}). 
	\end{align}
	
	In light of this observation, suppose, by absurd that $\varepsilon > 0$ is the minimum of \eqref{eq:zero}, being such value realizes for some matrix $\mathbf{U}$ and $\mathbf{V}$. Then, we can repeat the same construction and produce a pairs of matrix $\mathbf{A}$ and $\mathbf{B}$ such that $\Omega(\mathbf{A},\mathbf{B}) = \frac{\varepsilon}{2}$. Thus, necessarily, \eqref{eq:zero} holds being the objective non-negative. \endproof
	
	\subsection*{Proofs from Section 5: Matrix Dropout with Adaptive Dropout Rate}
	
	\begin{prop}\label{prop:adapt}
		
		For every $0 < \overline{\theta} < 1$, define \begin{equation}\label{eq:thetad}
		\theta(d) = \dfrac{\overline{\theta}}{d - (d - 1)\overline{\theta}}.
		\end{equation} Then, the following properties hold.
		\begin{enumerate}
			\item $0 < \theta(d) < 1$ for all $d \in \mathbb{N} \setminus \{ 0\}$. 
			\item $\dfrac{1 - \theta(kd)}{\theta(kd)} = k \dfrac{1-\theta(d)}{\theta(d)}$ for all $k \in \mathbb{N} \setminus \{ 0 \}$.
		\end{enumerate}
	\end{prop}
	
	\proof
	
	\begin{enumerate}[leftmargin=*]
		\item We will prove $\theta(d) > 0$ and $\theta(d) < 1$ separately. Since $\overline{\theta} > 0$, then $\theta(d) > 0$ if and only if $m - (m - 1)\overline{\theta} > 0$. But this is true since
		\begin{equation}
		m - (m - 1)\overline{\theta} = m - m\overline{\theta} + \overline{\theta} \geq m(1 - \overline{\theta}) > 0.
		\end{equation}
		On the other hand, since the fraction $\theta(d)$ is positive, $\theta(d) < 1$ is verified if and only if
		\begin{equation}
		\overline{\theta} < m - (m -1)\overline{\theta}
		\end{equation}
		if and only if
		\begin{equation}
		0 < m - m \overline{\theta}
		\end{equation}
		if and only if 
		\begin{equation}
		\overline{\theta} < 1
		\end{equation}
		which is actually true by assumption.
		\item The property can also be verified analytically by noticing that
		\begin{equation}
		\dfrac{1 - \theta(d)}{\theta(d)} = \dfrac{1 - \dfrac{\overline{\theta}}{d - (d - 1)\overline{\theta}}}{\dfrac{\overline{\theta}}{d - (d - 1)\overline{\theta}}} = \dfrac{d - (d - 1)\overline{\theta} - \overline{\theta}}{\overline{\theta}} = d \dfrac{1 - \overline{\theta}}{\overline{\theta}}. 
		\end{equation}
	\end{enumerate}
	
	This concludes the proof \endproof
	\begin{prop}
		For any $m \times n$ matrix $\mathbf{X}$, consider the expression 
		\begin{equation}\label{eq:norm}
		\| \mathbf{X} \|_{\hspace{-.5 mm} \vartriangle} = \min_{d,\U,\V} \sqrt{ \lambda_{d} \sum_{k = 1}^d \| \mathbf{u}_k \|_2^2 \| \mathbf{v}_k \|_2^2} \quad \text{s.t.} \quad 
		d \geq \rho(\X), \U \in \mathbb{R}^{m\times d}, \V \in \mathbb{R}^{n\times d} ~ \text{and}~ \U\V^\top = \X.
		\end{equation}
		where $\lambda_d = d \frac{1 - \overline{\theta}}{\overline{\theta}}$, for any $0 < \overline{\theta} < 1$, $\mathbf{u}_k \in \mathbb{R}^m$ and $\mathbf{v}_k \in \mathbb{R}^n$ define the $k$-th column in $\mathbf{U}$ and $\mathbf{V}$, respectively, $k = 1,\dots,d.$ Then, equation \eqref{eq:norm} defines a quasi-norm over $m \times n$ matrices.
	\end{prop}
	
	\proof Using the definition of quasi-norm, we have to prove the following
	\begin{align}
	& \| \mathbf{X} \|_\vartriangle \geq 0 \qquad \mbox{for every} \; \mathbf{X} \in \mathbb{R}^{m \times n} \label{eq:1} \\
	& \| \mathbf{X} \|_\vartriangle = 0  \Longleftrightarrow \mathbf{X} = \boldsymbol{0} \label{eq:2} \\
	& \| \alpha \mathbf{X} \|_\vartriangle = |\alpha|\|  \mathbf{X} \|_\vartriangle \qquad \mbox{for every} \; \alpha \in \mathbb{R} \; \mbox{and} \; \mathbf{X} \in \mathbb{R}^{m \times n} \label{eq:3} \\
	& \mbox{There exists $C >0$ such that} \; \| \mathbf{X} + \mathbf{Z} \|_\vartriangle \leq C (\|  \mathbf{X} \|_\vartriangle + \| \mathbf{Z} \|_{\vartriangle}) \; \mbox{for every} \; \mathbf{X},\mathbf{Z} \in \mathbb{R}^{m \times n}  \label{eq:4}
	\end{align}
	
	\begin{itemize}[leftmargin=*]
		\item \eqref{eq:1} $-$ Fix $\mathbf{X} \in \mathbb{R}^{m \times n}$ and arbitrary choose a pair of matrices $\mathbf{U}$ and $\mathbf{V}$, of suitable dimensions, such that $\mathbf{U}\mathbf{V}^\top = \mathbf{X}$. We get
		$$\sqrt{ \lambda_d \sum_{k = 1}^d \| \mathbf{u}_k \|_2^2 \| \mathbf{v}_k \|_2^2} \geq 0.$$ Since the very same holds when computing the minimum over $d,\mathbf{U}$ and $\mathbf{V}$, we obtain $\| \mathbf{X} \|_\vartriangle \geq 0.$
		\item \eqref{eq:2} ``$\| \mathbf{X} \|_\vartriangle = 0 \Rightarrow \mathbf{X} = \boldsymbol{0}$'' Let $\overline{\mathbf{U}} \in \mathbb{R}^{m \times \overline{d}}$ and $\overline{\mathbf{V}} \in \mathbb{R}^{n \times \overline{d}}$ such that $$\| \mathbf{X} \|_\vartriangle = \sqrt{ \lambda_d \sum_{k = 1}^d \| \overline{\mathbf{u}}_k \|_2^2 \| \overline{\mathbf{v}}_k \|_2^2}$$ and assume
		$$\sqrt{ \lambda_d \sum_{k = 1}^d \| \overline{\mathbf{u}}_k \|_2^2 \| \overline{\mathbf{v}}_k \|_2^2} = 0.$$
		Then
		$$\sum_{k = 1}^d \| \overline{\mathbf{u}}_k \|_2^2 \| \overline{\mathbf{v}}_k \|_2^2 = 0$$
		since $\lambda_d > 0$ (due to $0 < \theta(d) < 1$) and, also,
		\begin{equation}
		\| \overline{\mathbf{u}}_k \|_2^2 \| \overline{\mathbf{v}}_k \|_2^2 = 0 \qquad \mbox{for every} \; k = 1,\dots,d,
		\end{equation}
		since the summation is composed by non-negative terms. By using the zero-product property, we elicit
		\begin{equation}
		\mbox{for every} \; k = 1,\dots,d \qquad \| \overline{\mathbf{u}}_k \|_2^2 = 0 \; \mbox{or} \; \| \overline{\mathbf{v}}_k \|_2^2 = 0
		\end{equation}
		and
		\begin{equation}
		\mbox{for every} \; k = 1,\dots,d \qquad \| \overline{\mathbf{u}}_k \|_2 = 0 \; \mbox{or} \; \| \overline{\mathbf{v}}_k \|_2 = 0. 
		\end{equation}
		This implies that 
		\begin{equation}\label{eq:key}
		\mbox{for every} \; k = 1,\dots,d \qquad \overline{\mathbf{u}}_k = \boldsymbol{0} \; \mbox{or} \; \overline{\mathbf{v}}_k = \boldsymbol{0} 
		\end{equation}
		since $\| \cdot \|_2$ is a norm. But then, for any $i = 1,\dots,m$ and $j = 1,\dots,n$, the combination of the relationship
		\begin{equation}
		X_{ij} = \sum_{k = 1}^d \overline{U}_{ik} \overline{V}_{jk}
		\end{equation} 
		combined with \eqref{eq:key} gives
		\begin{equation}
		X_{ij} = 0 \quad \mbox{for every} \; i,j
		\end{equation}
		which is the thesis.
		\item \eqref{eq:2} ``$\| \mathbf{X} \|_\vartriangle = 0 \Leftarrow \mathbf{X} = \boldsymbol{0}$'' Assume $\mathbf{X} = \boldsymbol{0}$. Then the optimal decomposition $\mathbf{U}\mathbf{V}^\top = \mathbf{X}$ in the sense of \eqref{eq:norm} will be $\mathbf{U} = \boldsymbol{0}$ and $\mathbf{V} = \boldsymbol{0}$. This implies $\| \mathbf{X} \|_\vartriangle = 0$.
		\item \eqref{eq:3} (\emph{Absolute homogeneity}.) Since we already proved \eqref{eq:2}, we can skip the case $\alpha = 0$ because
		\begin{equation}
		\| 0 \mathbf{X} \|_\vartriangle = \| \boldsymbol{0} \|_\vartriangle \stackrel{\scriptsize \mbox{\eqref{eq:2}}}{=} 0 = 0 \cdot \| \mathbf{X} \|_\vartriangle.
		\end{equation} 
		Hence, let assume $\alpha \neq 0$. In such a case, by definition,
		\begin{equation}
		\| \alpha \mathbf{X} \|_\vartriangle = \minimize{\begin{matrix}
			d \geq \rho(\alpha \mathbf{X}) \\ \mathbf{U} \in \mathbb{R}^{m \times d} \\ \mathbf{V} \in \mathbb{R}^{n \times d} \\ {\rm s.t.}~\mathbf{U}\mathbf{V}^\top = \alpha \mathbf{X} \end{matrix}} \sqrt{ \lambda_d \sum_{k = 1}^d \| \mathbf{u}_k \|_2^2 \| \mathbf{v}_k \|_2^2}.
		\end{equation}
		Since $\alpha \neq 0$,
		\begin{equation}
		\| \alpha \mathbf{X} \|_\vartriangle = \minimize{\begin{matrix}
			d \geq \rho(\mathbf{X}) \\ \mathbf{U} \in \mathbb{R}^{m \times d} \\ \mathbf{V} \in \mathbb{R}^{n \times d} \\ {\rm s.t.}~\mathbf{U}\mathbf{V}^\top = \alpha \mathbf{X} \end{matrix}} \sqrt{ \lambda_d \sum_{k = 1}^d \| \mathbf{u}_k \|_2^2 \| \mathbf{v}_k \|_2^2}.
		\end{equation}
		Equivalently,
		\begin{equation}
		\| \alpha \mathbf{X} \|_\vartriangle = |\alpha |\minimize{\begin{matrix}
			d \geq \rho(\mathbf{X}) \\ \mathbf{U} \in \mathbb{R}^{m \times d} \\ \mathbf{V} \in \mathbb{R}^{n \times d} \\ {\rm s.t.}~(\frac{1}{\alpha} \mathbf{U})\mathbf{V}^\top = \mathbf{X} \end{matrix}} \sqrt{ \lambda_d \sum_{k = 1}^d \left \| \frac{1}{\alpha} \mathbf{u}_k \right \|_2^2 \| \mathbf{v}_k \|_2^2}.
		\end{equation}
		Since the transformation $\mathbf{U} \mapsto \widetilde{\mathbf{U}} := \frac{1}{\alpha} \mathbf{U}$ is invertible, we get
		\begin{equation}
		\| \alpha \mathbf{X} \|_\vartriangle = |\alpha |\minimize{\begin{matrix}
			d \geq \rho(\mathbf{X}) \\ \widetilde{\mathbf{U}} \in \mathbb{R}^{m \times d} \\ \mathbf{V} \in \mathbb{R}^{n \times d} \\ {\rm s.t.}~ \widetilde{\mathbf{U}}\mathbf{V}^\top = \mathbf{X} \end{matrix}} \sqrt{ \lambda_d \sum_{k = 1}^d \left \| \widetilde{\mathbf{u}}_k \right \|_2^2 \| \mathbf{v}_k \|_2^2} = |\alpha | \| \mathbf{X} \|_\vartriangle.
		\end{equation}
		\item \eqref{eq:4} $-$ (\emph{Generalized triangle inequality}.) Fix two arbitrary $m \times n$ matrices $\mathbf{X}$ and $\mathbf{Z}$. Let $\mathbf{U}_\mathbf{X} \in \mathbb{R}^{m \times d_\mathbf{X}},\mathbf{V}_\mathbf{X} \in \mathbb{R}^{n \times d_\mathbf{X}}$ the pairs of matrices which realize the minimum in $\| \mathbf{X} \|_\vartriangle$ and let $\mathbf{U}_\mathbf{Z} \in \mathbb{R}^{m \times d_\mathbf{Z}},\mathbf{V}_\mathbf{Z} \in \mathbb{R}^{n \times d_\mathbf{Z}}$ the same for $\| \mathbf{Z} \|_\vartriangle$. Define $\boldsymbol{\mathcal{U}} = [\mathbf{U}_\mathbf{X},\mathbf{U}_\mathbf{Z}]$ and $\boldsymbol{\mathcal{V}} = [\mathbf{V}_\mathbf{X},\mathbf{V}_\mathbf{Z}]$. Then,
		\begin{equation}
		\boldsymbol{\mathcal{U}} \boldsymbol{\mathcal{V}}^\top = \mathbf{U}_\mathbf{X}\mathbf{V}_\mathbf{X}^\top + \mathbf{U}_\mathbf{Z}\mathbf{V}_\mathbf{Z}^\top = \mathbf{X} + \mathbf{Z}
		\end{equation} 
		and notice that we can assume that $d_{\mathbf{X}} = d_{\mathbf{Z}} = d$. Indeed, in the arbitrary case, we can exploit the fact that $\lambda_{d_\mathbf{X} + d_\mathbf{Z}}$ can be bounded by $\lambda_{2 \max(d_\mathbf{X},d_\mathbf{Z})}$ and still apply the same reasoning. Therefore
		\begin{align*}
		\| \mathbf{X} + \mathbf{Z} \|_\vartriangle \leq \sqrt{ \lambda_{2d} \sum_{k = 1}^{2d} \| \boldsymbol{\mathcal{U}}_{:,k} \|_2^2 \| \boldsymbol{\mathcal{V}}_{:,k} \|_2^2},
		\end{align*}
		where the minimal value for $\| \mathbf{X} + \mathbf{Z} \|_\vartriangle$ induced by the optimal factorization, can be bounded by the analogous corresponding to $(\boldsymbol{\mathcal{U}},\boldsymbol{\mathcal{V}})$, each having $2d$ columns. Then,
		\begin{align*}
		\| \mathbf{X} + \mathbf{Z} \|_\vartriangle \leq \sqrt{ \lambda_{2d} \sum_{k = 1}^{d} \| [{\mathbf{u}_\mathbf{X}}]_k \|_2^2 \| [{\mathbf{v}_\mathbf{X}}]_k \|_2^2 + \lambda_{2d} \sum_{k = 1}^{d} \| [{\mathbf{u}_\mathbf{Z}}]_k \|_2^2 \| [{\mathbf{v}_\mathbf{Z}}]_k \|_2^2}.
		\end{align*}
		Since the square root is a sub-additive function,
		\begin{align*}
		\| \mathbf{X} + \mathbf{Z} \|_\vartriangle &\leq \sqrt{ \lambda_{2d} \sum_{k = 1}^{d} \| [{\mathbf{u}_\mathbf{X}}]_k \|_2^2 \| [{\mathbf{v}_\mathbf{X}}]_k \|_2^2} + \sqrt{ \lambda_{2d} \sum_{k = 1}^{d} \| [{\mathbf{u}_\mathbf{Z}}]_k \|_2^2 \| [{\mathbf{v}_\mathbf{Z}}]_k \|_2^2} \\ &= \sqrt{ 2 \lambda_{d} \sum_{k = 1}^{d} \| [{\mathbf{u}_\mathbf{X}}]_k \|_2^2 \| [{\mathbf{v}_\mathbf{X}}]_k \|_2^2} + \sqrt{ 2 \lambda_{d} \sum_{k = 1}^{d} \| [{\mathbf{u}_\mathbf{Z}}]_k \|_2^2 \| [{\mathbf{v}_\mathbf{Z}}]_k \|_2^2}.
		\end{align*}
		Exploiting the relationship $\lambda_{2d} = 2 \lambda_{d}$ and the definitions of $\mathbf{U}_{\mathbf{X}}$, $\mathbf{V}_{\mathbf{X}}$, $\mathbf{U}_{\mathbf{Z}}$ and $\mathbf{V}_{\mathbf{Z}}$. Then,
		\begin{align*}
		\| \mathbf{X} + \mathbf{Z} \|_\vartriangle &\leq \sqrt{ 2 \lambda_{d} \sum_{k = 1}^{d} \| [{\mathbf{u}_\mathbf{X}}]_k \|_2^2 \| [{\mathbf{v}_\mathbf{X}}]_k \|_2^2} + \sqrt{ 2 \lambda_{d} \sum_{k = 1}^{d} \| [{\mathbf{u}_\mathbf{Z}}]_k \|_2^2 \| [{\mathbf{v}_\mathbf{Z}}]_k \|_2^2} \\ & = \sqrt{2} (\| \mathbf{X} \|_\vartriangle + \| \mathbf{Z} \|_\vartriangle).
		\end{align*}
		We conclude by choosing $C := \sqrt{2}.$ \endproof
	\end{itemize}
	
	\begin{prop}
		The convex envelope of $\tfrac{1}{2} \| \mathbf{X}\|^2_\vartriangle$ is $\tfrac{1-\bar{\theta}}{2 \bar{\theta}} \| \mathbf{X}\|^2_\star$.
	\end{prop}
	
	\proof  First, recall that the convex envelope of a function $f$ is the largest closed, convex function $g$ such that $g(x) \leq f(x)$ for all $x$ and is given by $g = (f^*)^*$, where $f^*$ denotes the Fenchel dual of $f$, defined as $f^*(q) \equiv \sup_x \left<q,x\right>-f(x)$.  Let $\Theta(\mathbf{X}) = \tfrac{1}{2}\| \mathbf{X} \|_\vartriangle^2$, given by
	\begin{equation}
	\Theta(\mathbf{X}) = \inf_{\begin{matrix}
		d \geq \rho(\mathbf{X}) \\ \mathbf{U} \in \mathbb{R}^{m \times d} \\ \mathbf{V} \in \mathbb{R}^{n \times d} \\ {\rm s.t.}~\mathbf{U}\mathbf{V}^\top = \mathbf{X} \end{matrix}} \frac{\lambda_d}{2} \sum_{k = 1}^d \| \mathbf{u}_k \|_2^2 \| \mathbf{v}_k \|_2^2.
	\end{equation}
	and note that this can be equivalently written by the equation
	\begin{equation}
	\Theta( \mathbf{X} ) = \inf_{\begin{matrix}
		d \geq \rho(\mathbf{X}) \\ \mathbf{U} \in \mathbb{R}^{m \times d} \\ \mathbf{V} \in \mathbb{R}^{n \times d} \\ \Lambda \in \mathbb{R}^d \end{matrix}} \frac{\lambda_d}{2} \|\Lambda\|_2^2 \ \ {\rm s.t.} \ \ \sum_{k=1}^d \Lambda_k \mathbf{u}_k \mathbf{v}_k^T = \mathbf{X} \ \ {\rm and} \ \ (\| \mathbf{u}_k \|_2,\| \mathbf{v}_k \|_2) \leq (1,1) \ \ \forall k.
	\end{equation}
	This gives the Fenchel dual of $\Theta$ as
	\begin{equation}
	\label{eq:lambda_dot}
	\Theta^*(\mathbf{Q}) = \sup_d \sup_{\begin{matrix} \mathbf{U} \in \mathbb{R}^{m \times d} \\ \mathbf{V} \in \mathbb{R}^{n \times d} \\ \Lambda \in \mathbb{R}^d \end{matrix}} \sum_{k=1}^d \Lambda_k \left< \mathbf{Q},\mathbf{u}_k \mathbf{v}_k^T \right> - \frac{\lambda_d}{2} \|\Lambda\|_2^2 \ \ {\rm s.t.} \ \  (\| \mathbf{u}_k \|_2,\| \mathbf{v}_k \|_2) \leq (1,1) \ \ \forall k.
	\end{equation}
	Now, note that if we define the vector $\mathbf{B}_d(\mathbf{U},\mathbf{V}) \in \mathbb{R}^d$ as
	\begin{equation}
	\mathbf{B}_d(\mathbf{U},\mathbf{V}) = \begin{bmatrix} \left<\mathbf{Q},\mathbf{u}_1 \mathbf{v}_1^T \right> \\ \left<\mathbf{Q},\mathbf{u}_2 \mathbf{v}_2^T \right> \\ \vdots \\ \left<\mathbf{Q},\mathbf{u}_d \mathbf{v}_d^T \right> \end{bmatrix},
	\end{equation}
	then from \eqref{eq:lambda_dot} we have that
	\begin{align}
	\Theta^*(\mathbf{Q}) &= \sup_d \sup_{\begin{matrix} \mathbf{U} \in \mathbb{R}^{m \times d} \\ \mathbf{V} \in \mathbb{R}^{n \times d} \end{matrix}} \sup_{\Lambda \in \mathbb{R}^d} \left<\mathbf{B}_d(\mathbf{U},\mathbf{V}), \Lambda \right> - \frac{\lambda_d}{2} \|\Lambda\|_2^2 \ \ {\rm s.t.} \ \  (\| \mathbf{u}_k \|_2,\| \mathbf{v}_k \|_2) \leq (1,1) \ \ \forall k \\
	\label{eq:fenchel_B}
	&= \sup_d \sup_{\begin{matrix} \mathbf{U} \in \mathbb{R}^{m \times d} \\ \mathbf{V} \in \mathbb{R}^{n \times d} \end{matrix}} \frac{1}{2 \lambda_d} \|\mathbf{B}_d(\mathbf{U},\mathbf{V}) \|_2^2  \ \ {\rm s.t.} \ \  (\| \mathbf{u}_k \|_2,\| \mathbf{v}_k \|_2) \leq (1,1) \ \ \forall k.
	\end{align}
	where the final equality comes from noting that the supremum w.r.t. $\Lambda$ is the definition of the Fenchel dual of the squared $\ell_2$ norm evaluated at $\mathbf{B}_d(\mathbf{U},\mathbf{V})$.
	
	Now, from\eqref{eq:fenchel_B} and the definition of $\mathbf{B}_d(\mathbf{U},\mathbf{V})$ note that for a fixed value of $d$, \eqref{eq:fenchel_B} is optimized w.r.t. $(\mathbf{U},\mathbf{V})$ by choosing all the columns of $(\mathbf{U},\mathbf{V})$ to be equal to the maximum singular vector pair, given by
	\begin{equation}
	\sup_{\mathbf{u} \in \mathbb{R}^m, \mathbf{v} \in \mathbb{R}^n} \left< \mathbf{Q},\mathbf{u} \mathbf{v}^T \right> \ \ {\rm s.t.} \ \  (\| \mathbf{u} \|_2,\| \mathbf{v} \|_2) \leq (1,1).
	\end{equation}
	Note also that for this optimal choice of $(\mathbf{U},\mathbf{V})$ we have that $\mathbf{B}_d(\mathbf{U},\mathbf{V}) = \sigma( \mathbf{Q}) \mathbf{1}_d$ where $\sigma( \mathbf{Q})$ denotes the largest singular value of $\mathbf{Q}$ and $\mathbf{1}_d$ is a vector of all ones of size $d$.  Plugging this in \eqref{eq:fenchel_B} gives
	\begin{equation}
	\Theta^*(\mathbf{Q}) = \sup_d \frac{1}{2\lambda_d}\|\sigma( \mathbf{Q}) \mathbf{1}_d \|_2^2 = \sup_d \frac{\sigma^2( \mathbf{Q}) d}{2 \lambda_d} = \left(\frac{\bar{\theta}}{1-\bar{\theta}}\right) \frac{\sigma^2( \mathbf{Q})}{2},
	\end{equation}
	where recall $\lambda_d = d (1-\bar{\theta}) / \bar{\theta}$.  The result then follows by noting the well-known duality between the spectral norm (largest singular value) and the nuclear norm and basic properties of the Fenchel dual.
	\endproof
	
	\begin{prop}
		Let $\mathbf{X} = \mathbf{L} \boldsymbol{\Sigma} \mathbf{R}^\top$ be the singular value decomposition of $\mathbf{X}$. The optimal solution to
		\begin{equation}\label{eq:mid-term}
		\min_{\mathbf{Y}} \quad \| \mathbf{X} - \mathbf{Y} \|_F^2 + \lambda \| \mathbf{Y} \|_\star^2
		\end{equation}
		is given by $\mathbf{Y} = \mathbf{L} \S_{\mu}(\boldsymbol{\Sigma}) \mathbf{R}^\top$, where $\lambda > 0$, $\mu = \frac{\lambda d}{1 + \lambda d} \bar\sigma_d(\mathbf{X})$, $\bar\sigma_d(\mathbf{X})$ is the average of the top $d$ singular values of $\mathbf{X}$, $d$ represents the largest integer such that $\sigma_d(\mathbf{X}) > \frac{\lambda d}{1+\lambda d} \bar\sigma_d(\mathbf{X})$, and $\S_{\mu}$ is defined as the shrinkage thresholding operator which set to zero all singular values of $\mathbf{X}$ which are less or equal to $\mu$.
	\end{prop}
	
	\proof Since both the nuclear norm $\| \cdot \|_\star$ and the Frobenius norm $\| \cdot \|_F$ are rotationally invariant, up to non-restrictive rotations applied to the data matrix $\mathbf{X}$, the thesis can be equivalently proved by considering the following result.
	
	Let $\mathbf{x} = [x_1, \dots, x_r]$ a fixed vector with $x_i \geq x_{i+1} > 0$. Define $\mu_d$ as the average of the first $d$ entries of $\mathbf{x}$ Then, the optimal solution to the optimization problem
	\begin{equation}
	\min_{\mathbf{a} \in \mathbb{R}^r} \| \mathbf{a} - \mathbf{x} \|_2^2 + \lambda \| \mathbf{a} \|_1^2
	\end{equation}
	is given by $\mathbf{a} = [a_1,\dots,a_r]$ where
	\begin{equation}\label{eq:quelli}
	a_i = \begin{cases} x_i - \dfrac{\lambda d}{ 1 + \lambda d}\mu_d & i = 1,\dots,d \\ 0 & i = d + 1,\dots,r
	\end{cases} 
	\end{equation}
	where $d$ is the largest positive integer less or equal to $r$ such that all ai given in \eqref{eq:quelli} are positive. 
	
	In order to prove this claim, first note that the objective function is strictly convex and, hence, there is a unique global minimum. If $\lambda = 0$ the global minimizer is precisely $\mathbf{x}$, which is consistent with the formula given in the statement of the proposition. So, suppose that $ \lambda > 0$. Next, notice that if $\mathbf{a} = [a_1, a_2, \dots, a_r]$ is an optimal solution, then all $a_i$ must be non-negative. Indeed, if say $a_1 < 0$, then the vector $[-a_1, a_2, \dots, a_r]$ already gives a smaller
	objective value. Now, the first order optimality condition of our problem rewrites
	\begin{equation}\label{eq:sonno}
	\boldsymbol{0} \in (\mathbf{a} - \mathbf{x}) + \lambda \| \mathbf{a} \|_1 \partial \| \mathbf{a} \|_1. 
	\end{equation}
	There are two cases for each coordinate $i$ of \eqref{eq:sonno}.
	\begin{equation}\label{eq:zzz}
	\mbox{$a_i = x_i - \lambda \| \mathbf{a} \|_1$, if $a_i > 0$, and $x_i = \lambda \| \mathbf{a} \|_1 \xi_i,$ if $a_i = 0$.}
	\end{equation}
	where $\xi_i$ in \eqref{eq:zzz} is some number in the interval $[0, 1]$. Notice that since $x_i > 0$ for every $i$, the second condition in \eqref{eq:zzz} guarantees that the global solution can not be the zero vector, otherwise $\| \mathbf{a} \|_1 = 0$ and so $x_i = 0$ for every $i$. Thus, suppose that exactly the first $k \geq 1$ coordinates of $\mathbf{a}$ are non-zero. Then sum the equations $a_i = x_i - \lambda k \| \mathbf{a} \|_1$ for $i = 1, \dots, k$. We get
	\begin{equation}
	\| \mathbf{a} \|_1 = k \mu_k - \lambda k \|\mathbf{a} \|_1
	\end{equation}
	which gives
	\begin{equation}\label{eq:zzzz}
	\| \mathbf{a} \|_1 = \dfrac{k}{1 + \lambda k} \mu_k.
	\end{equation}
	Then \eqref{eq:zzz} and \eqref{eq:zzzz} give
	\begin{equation}\label{eq:Bday}
	\mbox{$a_i = x_i - \dfrac{\lambda k}{1 + \lambda k} \mu_k > 0$ for $i = 1,\dots,k$ and $a_i = 0$ for $i = k + 1,\dots,r.$}
	\end{equation}
	Now, let $d$ be the largest integer such that $a_i = x_i - \dfrac{\lambda d}{1 +  \lambda d} \mu_d > 0$ and define the vector
	\begin{equation}
	\mathbf{v} = \left[ x_1 - \dfrac{\lambda d}{1 +  \lambda d} \mu_d, \dots, x_d - \dfrac{\lambda d}{1 +  \lambda d} \mu_d, \underbrace{0, \dots, 0}_{\mbox{$r-d$ times}} \right].
	\end{equation}
	If $d = r$, then $\mathbf{v}$ satisfies the optimality condition \eqref{eq:zzz} and so it is the global minimizer. So suppose that $d < r$. In that case, to show that $\mathbf{v}$ is the global minimizer it suffices to show that
	\begin{equation}
	x_{d + 1} - \dfrac{\lambda d}{1 +  \lambda d} \mu_d \leq 0.
	\end{equation}
	since this is equivalent to saying that for any $i > $d there exists $\xi_i \in [0, 1]$ such that $x_i = \lambda \| \mathbf{v} \|_1 \xi_i$ in which case $\mathbf{v}$ satisfies the optimality condition \eqref{eq:zzz}. Now by the maximality of $d$, we have that
	\begin{equation}
	x_{d + 1} - \dfrac{\lambda (d + 1)}{1 +  \lambda (d + 1)} \mu_{d + 1} \leq 0.
	\end{equation}
	Equivalently, we get the following chain of inequalities
	\begin{align}
	\left( 1 - \dfrac{\lambda}{1 + \lambda(d + 1)}\right) x_{d + 1} - \dfrac{\lambda}{1 + \lambda(d +1 )} \sum_{k = 1}^d x_k \leq 0 \\
	\dfrac{1 + \lambda d}{1 + \lambda(d + 1)} x_{d + 1} - \dfrac{\lambda}{1 + \lambda(d +1 )} \sum_{k = 1}^d x_k  \leq 0 \\
	x_{d + 1} - \dfrac{\lambda d}{1 + \lambda d} \mu_d  \leq 0
	\end{align}
	from which we obtain the desired condition. \endproof

\end{document}